\documentclass{article}
\usepackage{ijcai11}

\usepackage{times}
\usepackage{amsmath,amsfonts,amssymb}
\usepackage{amsthm}

\usepackage{mathrsfs}
\usepackage[latin1]{inputenc}
\usepackage{graphicx}
\usepackage[ruled,vlined,linesnumbered]{algorithm2e}
\usepackage{url}
\usepackage{color}
\usepackage{gnuplottex}
\usepackage{rotating}
\usepackage{multirow}
\usepackage{subfig}
\usepackage{epsfig}
\usepackage{pifont}

\newtheorem{definition}{Definition}
\newtheorem{lemma}{Lemma}
\newtheorem{property}{Property}
\newtheorem{proposition}{Proposition}

\newcommand{\LNV}{{\sc{\small{Increasing\_Nvalue}}}}

\newcommand{\CH}{{\sc{\small{Change}}}}

\newcommand{\RE}{{\sc{\small{Regular}}}}
\newcommand{\CRE}{{\sc{\small{Cost-Regular}}}}

\newcommand{\SM}{{\sc{\small{Smooth}}}}
\newcommand{\SL}{{\sc{\small{Slide}}}}
\newcommand{\CP}{{\sc{\small{CardPath}}}}
\newcommand{\seqbin}{{\sc{\small{seq\_bin}}}}

\newcommand{\gstretch}{\emph{$\cstrStretch$-stretch}}
\newcommand{\gstretches}{\emph{$\cstrStretch$\nobreakdash-stretches}}

\newcommand{\vars}{{X}}
\newcommand{\ctrs}{{\cal C}}
\newcommand{\doms}{{\cal D}}

\newcommand{\cstrDef}{{C}}
\newcommand{\cstrStretch}{{C}}
\newcommand{\cstrBin}{{B}}

\DeclareMathOperator{\sumD}{\Sigma_{Di}}

\SetKw{False}{false}
\SetKw{True}{true}
\SetKw{Continue}{continue}
\SetKw{Break}{break}

\newtheorem{notation}{Notation}

\title{A Generalized Arc-Consistency Algorithm for a Class of Counting Constraints: Revised Edition that Incorporates One Correction \\ {\footnotesize This revised edition of the IJCAI'11 paper introduces the \emph{counting-continuous} notion (Def.~\ref{def:cc}), to correct  Prop.~\ref{prop:consistency}.}}

\author{Thierry Petit \and Nicolas Beldiceanu \and Xavier Lorca \\
Mines-Nantes, LINA UMR CNRS 6241, \\
  4, rue Alfred Kastler, FR-44307~Nantes, France. \\
  \{Thierry.Petit, Nicolas.Beldiceanu, Xavier.Lorca\}@mines-nantes.fr} 

\begin{document}
\maketitle

\begin{abstract} 
This paper introduces the \seqbin~meta-constraint with a polytime algorithm 
achieving generalized arc\nobreakdash-consistency according to some properties. \seqbin~can be used for encoding counting 
constraints such as \CH, \SM~or~\LNV. 
For some of these constraints and some of their variants GAC can be enforced with a time and space complexity linear in the sum of domain sizes, which 
improves or equals the best known results of the literature.
\end{abstract}

\section{Introduction}

Many constraints are such that a \emph{counting} variable is equal to the number of times a given property is satisfied in a sequence of variables. To represent some of these constraints in a generic way, we introduce the \seqbin$(N,X,\cstrStretch,\cstrBin)$ meta-constraint, where $N$ is an integer variable, $X$ is a sequence of integer variables and $\cstrStretch$ and $\cstrBin$ are two binary constraints. 

Based on the notion $\cstrStretch$-stretch, a generalization of stretch~\cite{Pesant.CP.2001} where the equality constraint is made explicit and is replaced by $\cstrStretch$, \seqbin{} holds if and only if two conditions are both satisfied: (1)~$N$ is equal to  the number of 
$\cstrStretch$-stretches in the sequence $X$, and (2)~$\cstrBin$ holds on any pair of consecutive variables in $X$.

Among the constraints that can be expressed thanks to \seqbin, many were introduced for solving real-world problems, \emph{e.g.},  \CH~\cite{Cosytec97} (time tabling problems), \SM~\cite{beldiceanu10} (time tabling and scheduling), or \LNV~\cite{BeldiceanuHermenierLorcaPetit10} (symmetry breaking for resource allocation problems). 

The main contribution of this paper is a generic polytime filtering algorithm for~\seqbin{}, which achieves generalized arc\nobreakdash-consistency (GAC) 
according to some conditions on $B$ and $C$. 
This algorithm can be seen as a generalization of the \LNV~filtering algorithm~\cite{BeldiceanuHermenierLorcaPetit10}. 
Given $n$ the size of $X$, $d$ the maximum domain size, and $\sumD$ the sum of domain sizes, 
%
we characterize properties on $\cstrStretch$ and $\cstrBin$ which lead to a time and space complexity in $O(\sumD)$. 
These properties are satisfied when \seqbin~represents~
\LNV, and several variants of~\CH~(provided its parameter is a monotonic binary constraint, \emph{e.g.}, '$\leq$', '$<$', '$\geq$', '$>$').  For these constraints, our technique improves or equals the best known results.

Section~\ref{prelem} provides the definitions used in this paper.
Section~\ref{def:constraintClass} defines \seqbin{} and shows how to express well-known constraints with \seqbin. 
Section~\ref{sec:theory} provides a necessary and sufficient condition for achieving GAC.
Section~\ref{sec:algos} details the corresponding GAC filtering algorithm.
Finally, Section~\ref{relatedWorks} discusses about related works and Section~\ref{conclusion} concludes.

\section{Background} \label{prelem}
A {\em Constraint Network} is defined by a sequence of variables $\vars = [x_0, x_1, \ldots,x_{n-1}]$, a sequence of domains $\doms$, where each $D(x_i) \in \doms$ is the finite set of values that variable $x_i$ can take, and a set of constraints $\ctrs$ that specifies the allowed combinations of values for given subsets of variables. $\min(x)$ (resp. $\max(x)$) is the minimum (resp. maximum) value of $D(x)$. 
A sequence of variables $X'=[x_i,x_{i+1},\ldots,x_j]$, $0 \leq i \leq j \leq n-1$ (resp. $i > 0$ or $i < n-1$), is a \emph{subsequence} (resp. a \emph{strict subsequence}) of $X$ and is denoted by $X' \subseteq X$ (resp. $X' \subset X$). 
$A[\vars]$ denotes an assignment of values to variables in $\vars$. Given $x \in X$, $A[x]$ is the value of $x$ in $A[\vars]$. $A[\vars]$ is \emph{valid} if and only if $\forall x_i \in \vars$, $A[x_i] \in D(x_i)$. An \emph{instantiation} $I[\vars]$ is a valid assignment of $\vars$. Given $x \in \vars$, $I[x]$ is the value of $x$ in $I[\vars]$. Given the sequence $\vars$ and $i$, $j$ two integers such that $0\leq i\leq j\leq n-1$, $I[x_i,\ldots,x_j]$ is the projection of $I[\vars]$ on $[x_i,x_{i+1},,\ldots,x_j]$.
A \emph{constraint} $\cstrDef(\vars) \in \ctrs$ specifies the allowed combinations of values for $\vars$. We also use the simple notation $\cstrDef$. $\cstrDef(\vars)$ defines a subset ${\cal R}_\cstrDef({\cal D})$ of the cartesian product of the domains  $\Pi_{x_i\in X}D(x_i)$. If $\vars$ is a pair of variables, then $\cstrDef(\vars)$ is \emph{binary}. We denote by $v \cstrDef w$ a pair of values $(v,w)$ that satisfies a binary constraint $\cstrDef$. 
$\neg \cstrDef$ is the \emph{opposite} of $\cstrDef$, that is, $\neg \cstrDef$ defines the relation ${\cal R}_{\neg \cstrDef}({\cal D})$ = $\Pi_{x_i\in X}D(x_i)$ $\setminus$ ${\cal R}_\cstrDef({\cal D})$. 
A \emph{feasible instantiation} $I[\vars]$ of  $\cstrDef[\vars]$ is an instantiation which is in ${\cal R}_\cstrDef({\cal D})$. We say that $I[\vars]$ \emph{satisfies} $\cstrDef(\vars)$, or that $I[\vars]$ is a {\em support} on $\cstrDef(\vars)$. 
Otherwise, $I[\vars]$ \emph{violates} $\cstrDef(\vars)$. If $C$ is a binary constraint on $X=\{x_i, x_{i+1}\}$ and $v \in D(x_i)$ then the set of supports such that $x_i = v$ can be considered as a set of values (a subset of $D(x_{i+1}))$. 
A \emph{solution} of a constraint network is an instantiation of all the variables satisfying all the constraints.

Value $v \in D(x_i)$, $x_i \in X$, is (generalized) {\em arc\--consistent} (GAC) with respect to~$\cstrDef(\vars)$ if and only if $v$ belongs to a support of $C(\vars)$.  A domain $D(x_i)$, $x_i \in X$, is GAC with respect to~$\cstrDef(\vars)$ if and only if $\forall v\in D(x_i)$, $v$ is GAC with respect to~$\cstrDef(\vars)$. $\cstrDef(\vars)$ is GAC if and only if $\forall x_i\in X$, $D(x_i)$ is GAC with respect to~$\cstrDef(\vars)$.
A constraint network is GAC if and only if it is closed for \emph{GAC}~\cite{Bessiere06}: $\forall x_i \in \vars$ all values in $D(x_i)$ that are not {GAC} with respect to a constraint in $\mathcal{C}$ have been removed.

\section{The {\sc\small SEQ\_BIN}~Meta-Constraint} \label{def:constraintClass}

We first generalize the notion of \emph{stretches}~\cite{Pesant.CP.2001} to characterize a sequence of consecutive variables where the same binary constraint is satisfied. 

\begin{definition}[\gstretch]\label{def:c-stretch}
Let $I[X]$ be an instantiation of the variable sequence $X = [x_0,x_1,\ldots,x_{n-1}]$ and  \cstrStretch{} a binary constraint. The \cstrStretch-sequence constraint $\mathcal{C}(I[X],\cstrStretch)$ 
holds if and only if: 
\begin{itemize}
\item Either $n=1$, 
\item or $n>1$ and $\forall k \in [0, n-2]$~$\cstrStretch(I[x_k],I[x_{k+1}])$ holds. 
\end{itemize}
A \gstretch{} of $I[X]$ is a subsequence $X'$ $\subseteq$ $X$ such that the two following conditions are both satisfied:
\begin{enumerate}
	\item The \cstrStretch-sequence $\mathcal{C}(I[X'],\cstrStretch)$ holds,
	\item $\forall X"$ such that $X' \subset X" \subseteq X$ the \cstrStretch-sequence $\mathcal{C}(I[X"],\cstrStretch)$ does not hold.
\end{enumerate}
\end{definition}

The intuition behind Definition~\ref{def:c-stretch} is to consider the maximum length subsequences  
where the binary constraint $\cstrStretch$ is satisfied between consecutive variables.
Thanks to this generalized definition of stretches we can now introduce \seqbin. 
\begin{definition}\label{def:ctr}
The meta-constraint \seqbin$(N,X,\cstrStretch,\cstrBin)$ is defined by a variable $N$,  a sequence of $n$ variables $X = [x_0,x_1,\ldots,x_{n-1}]$ and two binary constraints $\cstrStretch$ and $\cstrBin$. Given an instantiation $I[N, x_0,x_1,\ldots,x_{n-1}]$, \seqbin$(N,X,\cstrStretch,\cstrBin)$ is satisfied if and only if for any $ i\in[0,n-2]$,  $I[x_i]$\,$\cstrBin$\,$I[x_{i+1}]$ holds, and $I[N]$ is equal to the number of \gstretches{} in $I[X]$.
\end{definition}

The constraint \CH~was introduced in the context of timetabling problems~\cite{Cosytec97}, in order to put an upper limit on the number of changes of job types during a given period. The relation between classical stretches and \CH~was initially stressed in~\cite[page 64]{Hellsten04}.
\emph{\CH}  is defined on a variable $N$, a sequence of variables $X = [x_0,x_1,\ldots,x_{n-1}]$, and a binary constraint $C \in \{=,\neq,<,>,\leq,\geq\}$. It is satisfied if and only if $N$ is equal to the number of times the constraint $C$ holds on consecutive variables of $X$.
Without hindering propagation (the constraint network is Berge\--acyclic),  \CH~can be reformulated as \seqbin$(N',X,\neg C,$\texttt{true}$) \wedge [N = N'-1]$, where \texttt{true} is the universal constraint.

%
%

\SM$(N,X)$ is a variant of \CH$(N,X,C)$, where $x_i$\,$C$\,$x_{i+1}$ is defined by $|x_i-x_{i+1}| > \mathit{cst}$, $\mathit{cst} \in \mathbb{N}$. It is useful to limit  the number of drastic variations on a cumulative profile~\cite{beldiceanu10,dec10}.

As a last example, consider the \LNV~constraint, which is a specialized version of {\sc NValue}~\cite{Pachet99}. It was introduced for breaking variable symmetry in the context of resource allocation problems~\cite{BeldiceanuHermenierLorcaPetit10}. 
\LNV~is defined on a variable $N$ and on a sequence of variables $X =[x_0,x_1,\ldots,x_{n-1}]$. Given an instantiation, \emph{\LNV$(N,X)$} is satisfied if and only if $N$ is equal to the number of distinct values assigned to variables in $X$, and for any $i\in[0,n-2]$, $x_i\leq x_{i+1}$.  
We reformulate \LNV$(N,X)${} as \seqbin$(N,X,=,\leq)$. 

\section{Consistency of {\sc\small SEQ\_BIN}}\label{sec:theory}
We first present how to compute, for any value in a given domain of a variable $x_i  \in X$, the minimum and maximum number of $C$-stretches within the suffix of $X$ starting at $x_i$  (resp. the prefix of $X$ ending at $x_i$) satisfying a chain of binary constraints of type $\cstrBin$. Then, we introduce several properties useful to obtain a feasibility condition for \seqbin, and a necessary and sufficient condition for filtering which leads to the GAC filtering algorithm presented in Section~\ref{sec:algos}.

\subsection{Computing of the Number of $C$\--stretches} \label{computeCStretches}
According to Definition~\ref{def:ctr}, we have to ensure that the chain of $\cstrBin$ constraints are satisfied  along the sequence of variables $X = [x_0, x_1,\ldots, x_{n-1}]$. An instantiation $I[X]$ is said \emph{$\cstrBin$\nobreakdash-coherent} if and only if either $n=1$ or for any $i \in [0,n-2]$, we have $I[x_i]$\,$\cstrBin$\,$I[x_{i+1}]$. A value $v\in D(x_i)$ is said to be \emph{$\cstrBin$\nobreakdash-coherent} with respect to $x_i$ if and only if it can be part of at least one $\cstrBin$-coherent instantiation. Then, given an integer $i \in [0, n-2]$, if $v \in D(x_i)$ is $\cstrBin$-coherent \emph{with respect to} $x_i$ then there exists $w \in D(x_{i+1})$ such that $v\,\cstrBin\,w$. 

Consequently, within a given domain $D(x_i)$, values that are not $\cstrBin$\nobreakdash-coherent can be removed since they cannot be part of any solution of \seqbin. Our aim is now to compute for each $\cstrBin$\--coherent value $v$ in the domain of any variable $x_i$ the minimum and maximum number of \gstretches{} on $X$. 
\begin{notation}
$\underline{s}(x_{i},v)$ (resp. $\overline{s}(x_{i},v)$) is the minimum (resp. maximum) number of $C$-stretches within the sequence of variables $[x_i, x_{i+1}, \ldots, x_{n-1}]$ under the hypothesis that $x_i = v$. $\underline{p}(x_{i},v)$ (resp. $\overline{p}(x_{i},v)$) is the minimum (resp. maximum) number of $C$-stretches within the sequence $[x_0, x_1, \ldots, x_{i}]$ under the hypothesis that $x_i = v$. 
Given $X =[x_0,x_1,\ldots,x_{n-1}]$, $\underline{s}(X)$ (resp. $\overline{s}(X)$) denotes the minimum (resp. maximum) value of $\underline{s}(x_0,v)$ (resp. $\overline{s}(x_0,v)$). 
\end{notation}

\begin{lemma}\label{lem:Bmin}
Given \seqbin$(N,X,\cstrStretch,\cstrBin)$ with $X=[x_0,x_1,\ldots,x_{n-1}]$, assume the domains in $X$ contain only $\cstrBin$\nobreakdash-coherent values. 
Given $i \in [0, n-1]$ and $v \in D(x_i)$,
\begin{itemize}
\item If $i = n-1$:  $\underline{s}(x_{n-1},v) = 1$.
\item {Else: \footnotesize
$$
\underline{s}(x_i,v)= \min_{w \in D(x_{i+1})}
\begin{pmatrix}
	\min_{[v \cstrBin w] \wedge [v \cstrStretch w]}(\underline{s}(x_{i+1},w)), \\
	\min_{[v \cstrBin w] \wedge [v \neg \cstrStretch w] }(\underline{s}(x_{i+1},w))+1
\end{pmatrix}
$$
\normalsize
}
\end{itemize}
\end{lemma}
\begin{proof}
By induction. From Definition~\ref{def:c-stretch}, for any $v\in D(x_{n-1})$, we have $\underline{s}(x_{n-1},v)=1$ (i.e., a \gstretch{} of length $1$). Consider now $x_i \in X$ with  $i<n-1$, and a value $v\in D(x_i)$. Consider the set of instantiations 
$I[x_{i+1},x_{i+2},\ldots,x_{n-1}]$ that are $\cstrBin$-coherent, and that minimize the number of \gstretches{} in $[x_{i+1},x_{i+2},\ldots,x_{n-1}]$. We denote this minimum number of \gstretches{} by $\mathit{mins}$. At least one $\cstrBin$-coherent instantiation exists since all values in the domains of $[x_{i+1},x_{i+2},\ldots,x_{n-1}]$ are $\cstrBin$-coherent. For each such instantiation, let us denote by $w$ the value associated with $I[x_{i+1}]$.
Either there exists such an instantiation with $\mathit{mins}$ \gstretches{} with the conjunction $\cstrBin \wedge \cstrStretch$ satisfied by $(I[x_i],I[x_{i+1}])$. Then, $\underline{s}(x_i,v) =\underline{s}(x_{i+1},w)$ since the first \gstretch{} of $I[x_{i+1},x_{i+2},\ldots,x_{n-1}]$ is extended when augmenting $I[x_{i+1},x_{i+2},\ldots,x_{n-1}]$ with value $v$ for $x_i$. 
Or all instantiations $I[x_{i+1},x_{i+2},\ldots,x_{n-1}]$ with $\mathit{mins}$ \gstretches{} are such that $\cstrStretch$ is violated by $(I[x_i],I[x_{i+1}])$: $(I[x_i],I[x_{i+1}])$ satisfies $\cstrBin \wedge \neg\cstrStretch$. By construction, any instantiation $I[x_{i},x_{i+1},\ldots,x_{n-1}]$ with $I[x_i] = v$ has a number of \gstretches{} strictly greater than $\mathit{mins}$. Consequently, given $I[x_{i+1}, x_{i+2},\ldots,x_{n-1}]$ with $\mathit{mins}$ \gstretches{}, the number of \gstretches{} obtained by augmenting this instantiation with value $v$ for $x_i$ is exactly $\mathit{mins} + 1$. 
\end{proof}
\begin{lemma}\label{lem:Bmax}
Given \seqbin$(N,X,\cstrStretch,\cstrBin)$ with $X=[x_0,x_1,\ldots,x_{n-1}]$, assume the domains in $X$ contain only $\cstrBin$\nobreakdash-coherent values. 
Given $i \in [0, n-1]$ and $v \in D(x_i)$: 
\begin{itemize}
\item If $i = n-1$: $\overline{s}(x_{n-1},v) = 1$.  
\item {Else: \footnotesize
$$
\overline{s}(x_i,v)= \max_{w \in D(x_{i+1})}
\begin{pmatrix}
	\max_{[v \cstrBin w] \wedge [v \cstrStretch w]}(\overline{s}(x_{i+1},w)), \\
	\max_{[v \cstrBin w] \wedge [v \neg \cstrStretch w]}(\overline{s}(x_{i+1},w))+1
\end{pmatrix}
$$
\normalsize
}
\end{itemize}
\end{lemma}

Given a sequence of variables $[x_0,x_1,\ldots,x_{n-1}]$ such that their domains contain only $\cstrBin$-coherent values, for any $x_i$ in the sequence and any $v \in D(x_i)$, computing $\underline{p}(x_i,v)$ (resp. $\overline{p}(x_i,v)$) is symmetrical to $\underline{s}(x_i,v)$ (resp. $\overline{s}(x_i,v)$). We substitute $\min$ by $\max$ (resp. $\max$ by $\min$), $x_{i+1}$ by $x_{i-1}$, and $vRw$ by $wRv$ for any $R \in \{\cstrBin, \cstrStretch, \neg \cstrStretch\}$.

\subsection{Properties on the Number of $C$\--stretches} \label{propCStretches}
This section provides the properties linking the values in a domain $D(x_i)$ with the minimum and maximum number of $C$\nobreakdash-stretches in $X$. We  consider only $\cstrBin$\nobreakdash-coherent values, which may be part of a feasible instantiation of \seqbin{}.
Next property is a direct consequence of Lemmas~\ref{lem:Bmin} and~\ref{lem:Bmax}. 
\begin{property}\label{property:minleqmax}
For any $\cstrBin$-coherent value $v$ in $D(x_i)$, with respect to $x_i$, $\underline{s}(x_i,v)\leq\overline{s}(x_i,v) $. 
\end{property}
\begin{property}\label{prop:minmax}
Consider \seqbin$(N,X,\cstrStretch,\cstrBin)$, a variable $x_i \in X$ ($0 \leq i \leq n-1$), and two $\cstrBin$-coherent values $v_1,v_2 \in D(x_i)$. 
If $i=n-1$ or if there exists a $\cstrBin$-coherent $w \in D(x_{i+1})$ such that $v_1 \cstrBin w$ and $v_2 \cstrBin w$, then $\overline{s}(x_i,v_1) +1 \geq \underline{s}(x_i,v_2)$.  
\end{property}
\begin{proof} 
Obviously, if $i=n-1$. If $v_1 = v_2$, by Property~\ref{property:minleqmax} the property holds. Otherwise, assume there exist two values $v_1$ and $v_2$ such that $\exists w \in D(x_{i+1})$ for which $v_1 \cstrBin w$ and $v_2 \cstrBin w$, and $\overline{s}(x_i,v_1) +1 < \underline{s}(x_i,v_2)$ (hypothesis~$H$). 
By Lemma~\ref{lem:Bmax}, $\overline{s}(x_i,v_1) \geq \overline{s}(x_{i+1},w)$. 
By Lemma~\ref{lem:Bmin},  $\underline{s}(x_i,v_2) \leq \underline{s}(x_{i+1},w)+1$. 
From hypothesis $H$, this entails $\overline{s}(x_{i+1},w)+1 < \underline{s}(x_{i+1},w)+1$, which leads to 
$\overline{s}(x_{i+1},w) < \underline{s}(x_{i+1},w)$, which is, by  Property~\ref{property:minleqmax}, 
not possible.  
\end{proof}

\begin{property}\label{prop:nondisjoint}
Consider \seqbin$(N,X,\cstrStretch,\cstrBin)$, a variable $x_i \in X$ ($0 \leq i \leq n-1$), and two $\cstrBin$-coherent values $v_1,v_2 \in D(x_i)$. 
If either $i = n-1$ or there exists $\cstrBin$-coherent $w \in D(x_{i+1})$ such that $v_1\,\cstrBin\,w$ and $v_2\,\cstrBin\,w$ then, for any $k \in [\min(\underline{s}(x_i,v_1),\underline{s}(x_i,v_2)),\max(\overline{s}(x_i,v_1),\overline{s}(x_i,v_2))]$, either $k \in [\underline{s}(x_i,v_1),\overline{s}(x_i,v_1)]$ or $k \in [\underline{s}(x_i,v_2),\overline{s}(x_i,v_2)]$.
\end{property}
\begin{proof} Obviously, if $i=n-1$ or $v_1 = v_2$ .
 If  $[\underline{s}(x_i,v_1),\overline{s}(x_i,v_1)] \cap [\underline{s}(x_i,v_2),\overline{s}(x_i,v_2)]$ is not empty, then the property holds. Assume $[\underline{s}(x_i,v_1),\overline{s}(x_i,v_1)]$ and $[\underline{s}(x_i,v_2),\overline{s}(x_i,v_2)]$ are disjoint. W.l.o.g., assume $\overline{s}(x_i,v_1) < \underline{s}(x_i,v_2)$. By Property~\ref{prop:minmax}, $\overline{s}(x_i,v_1) +1 \geq \underline{s}(x_i,v_2)$, thus $\overline{s}(x_i,v_1) = \underline{s}(x_i,v_2)-1$. Either $k \in [\underline{s}(x_i,v_1),\overline{s}(x_i,v_1)]$ or $k \in [\underline{s}(x_i,v_2),\overline{s}(x_i,v_2)]$ (there is no hole in the range formed by the union of these intervals). 
\end{proof}

\subsection{Properties on Binary Constraints} \label{neighborSubst}
Property~\ref{prop:nondisjoint} is central for providing a GAC filtering algorithm based on the count,
for each $B$-coherent value in a domain, of the minimum and maximum number of \gstretches~in complete instantiations. 
Given \seqbin$(N,X,\cstrStretch,\cstrBin)$, we focus on binary constraints $B$ which guarantee that Property~\ref{prop:nondisjoint} holds. 

\begin{definition}\emph{\cite{vandevten92}}~\label{def:subst}
A binary constraint $F$ is \emph{monotonic} if and only if there exists a total ordering $\prec$ of values in domains such that: for any value $v$ and any value $w$,  
$v F w$ holds implies $v' F w'$ holds for all valid tuple such that $v' {\prec}$ $v$ and $w {\prec}$ $w'$. 
\end{definition}

Binary constraints $<$, $>$, $\leq$ and $\geq$ are monotonic, as well as the universal constraint \texttt{true}. 

\begin{property}\label{prop:subst}
Consider \seqbin$(N,X,\cstrStretch,\cstrBin)$ such that all non $\cstrBin$-coherent values have been removed from domains of variables in $X$. $\cstrBin$~is monotonic if and only if
for any variable $x_i \in \vars$, $0 \leq i < n-1$, for any values $v_1, v_2 \in D(x_i)$,  there exists $w \in D(x_{i+1})$ such that $v_1 \cstrBin w$ and $v_2 \cstrBin w$. 
\end {property}
\begin{proof}
($\Rightarrow$) From Definition~\ref{def:subst} and since we consider only $\cstrBin$-coherent values, each value has at least one support on $\cstrBin$.
Moreover, from Definition~\ref{def:subst}, $\{w \mid v_2 \cstrDef w\} \subseteq \{w \mid v_1 \cstrDef w\}$ or $\{w \mid v_1 \cstrDef w\} \subseteq \{w \mid v_2 \cstrDef w\}$. The property holds.
($\Leftarrow$) Suppose that the second proposition is true (hypothesis H) and  $\cstrBin$~is not monotonic. From Definition~\ref{def:subst}, if $\cstrBin$ is not monotonic then  $\exists v_1$ and $v_2$ in the domain of a variable $x_i \in X$ such that, by considering the constraint $\cstrBin$ on the pair of variables ($x_i, x_{i+1}$), neither $\{w \mid v_2 \cstrDef w\} \subseteq \{w \mid v_1 \cstrDef w\}$ nor $\{w \mid v_1 \cstrDef w\} \subseteq \{w \mid v_2 \cstrDef w\}$. Thus, there exists a support $v_1\cstrBin w$ such that $(v_2,w)$ is not a support on $\cstrBin$, and  a support $v_2\cstrBin w'$ such that $(v_1,w')$ is not a support on $\cstrBin$. We can have $D(x_{i+1}) = \{w,w'\}$, which leads to a contradiction with H. The property holds.  
\end{proof}

\subsection{Feasibility} \label{feasibility}
From Property~\ref{prop:subst}, this section provides an equivalence relation between the existence of a solution for \seqbin~and the current variable domains of $X$ and $N$. 
Without loss of generality, in this section we consider that all non $\cstrBin$-coherent values have been removed from domains of variables in $X$. 
First, Definition~\ref{def:ctr} entails the following necessary condition for feasibility. 

\begin{proposition}\label{prop:necessary}
Given \seqbin$(N,X,\cstrStretch,\cstrBin)$, if $\underline{s}(X)>\max(D(N))$ or $\overline{s}(X)<\min(D(N))$ then \seqbin{} fails.
\end{proposition} 

$D(N)$ can be restricted to $[\underline{s}(X),\overline{s}(X)]$, but $D(N)$ may have holes or may be strictly included in $[\underline{s}(X),\overline{s}(X)]$.
We have the following proposition. 

\begin{proposition} \label{consistentK}
Consider~\seqbin$(N,X,\cstrStretch,\cstrBin)$ such that $\cstrBin$ is monotonic, with $X = [x_0,x_1,\ldots,x_{n-1}]$. For any integer $k$ in $[\underline{s}(X), \overline{s}(X)]$ there exists $v$ in $D(x_0)$ such that $k\in[\underline{s}(x_0,v),\overline{s}(x_0,v)]$.
\end{proposition}
\begin{proof}
Let $v_1 \in D(x_0)$ a value such that $\underline{s}(x_0,v_1)=\underline{s}(X)$.
Let $v_2 \in D(x_0)$ a value such that $\overline{s}(x_0,v_2)=\overline{s}(X)$.  
By Property~\ref{prop:subst}, either $n=1$ or $\exists w \in D(x_{1})$ such that $v_1 \cstrBin w$ and $v_2 \cstrBin w$. 
Thus, from Property~\ref{prop:nondisjoint},
$\forall k \in [ \underline{s}(X),\overline{s}(X)]$, either $k \in [\underline{s}(x_0,v_1),\overline{s}(x_0,v_1)]$ or $k \in [\underline{s}(x_0,v_2),\overline{s}(x_0,v_2)]$. 
\end{proof}
By Proposition~\ref{consistentK}, any value for $N$ in $D(N) \cap [\underline{s}(X),\overline{s}(X)]$ is generalized arc-consistent provided a property is satisfied on the instance of~\seqbin~we consider:  given a variable $x_i$, for any value $v$ in $D(x_i)$ and for all $k \in [\underline{s}(x_i,v),\overline{s}(x_i,v)]$, there exists a solution of ~\seqbin$(N,X,\cstrStretch,\cstrBin)$ with exactly $k$ $C$-stretches.   

\begin{definition}\label{def:cc}
The constraint~\seqbin$(N,X,\cstrStretch,\cstrBin)$ is \emph{counting-continuous} if and only if for any instantiation $I[X]$ with $k$ $C$-stretches, for any variable $x_i \in X$, changing the value of $x_i$ in $I[X]$ leads to a number of $C$-stretches equal either to $k$, or to $k+1$, or to $k-1$.   
\end{definition}

\begin{property}
Consider~\seqbin$(N,X,\cstrStretch,\cstrBin)$ such that $\cstrBin$ is monotonic, with $X = [x_0,x_1,\ldots,x_{n-1}]$, $x_i$ a variable and $v \in D(x_i)$. If~\seqbin$(N,X,\cstrStretch,\cstrBin)$ is counting-continuous then there exists for any integer $k \in[\underline{s}(x_i,v),\overline{s}(x_i,v)]$ an instantiation $I[\{x_i, \ldots, x_{n-1}\}]$ with exactly $k$ $C$-stretches.
\end{property}
\begin{proof} 
By recurrence, we assume that the property is true for all instantiations of $[x_j,\ldots,x_{n-1}]$ such that $j>i$ (the property is obviously true if $j=n-1$). At step $i$, we assume that 
there exists $k \in[\underline{s}(x_i,v),\overline{s}(x_i,v)]$ such that there is no instantiation $I[\{x_i, \ldots, x_{n-1}\}]$ with $k$ $C$-stretches, while~\seqbin$(N,X,\cstrStretch,\cstrBin)$ is counting-continuous  (hypothesis). We prove that this assumption leads to a contradiction. 
By Lemmas~\ref{lem:Bmin} and~\ref{lem:Bmax}, there exists an instantiation $I'[\{x_i, \ldots, x_{n-1}\}]$ with $\underline{s}(x_i,v)$ $C$-streches and $I'[x_i]=v$, and there exists an instantiation $I''[\{x_i, \ldots, x_{n-1}\}]$ with $\overline{s}(x_i,v)$ $C$-stretches and $I''[x_i]=v$. Thus, by hypothesis, $k > \underline{s}(x_i,v)$ and $k < \overline{s}(x_i,v)$. 
We have $\overline{s}(x_i,v) \geq \underline{s}(x_i,v)+2$. By Property~\ref{prop:nondisjoint} and since the property is assumed true for all instantiations of $[x_{i+1},\ldots,x_{n-1}]$, 
there exists at least one pair of values $w_1$, $w_2$ in $D(x_{i+1})$ such that $\underline{s}(x_{i+1},w_2) = \overline{s}(x_{i+1},w_1)+1$, $(v,w_1)$ satisfies $C$ and $(v,w_2)$ violates $C$ (this is the only possible configuration leading to the hypothesis). In this case \seqbin$(N,X,\cstrStretch,\cstrBin)$ is not counting-continuous: Given an instantiation $I[x_i, \ldots, x_{n-1}]$ with $I[x_i]=v$ and $I[x_{i+1}] = w_1$, changing $w_1$ by $w_2$ for $x_{i+1}$ increases the number of $C$-stretches by $2$.
\end{proof}

\begin{proposition}\label{prop:consistency}
Given an instance of~\seqbin$(N,X,\cstrStretch,\cstrBin)$ which is counting-continuous and such that $\cstrBin$ is monotonic, \seqbin$(N,X,\cstrStretch,\cstrBin)$ has a solution if and only if $[\underline{s}(X),\overline{s}(X)]\cap D(N)\neq\emptyset$.
\end{proposition}

\begin{proof} ($\Rightarrow$) Assume \seqbin$(N,X,\cstrStretch,\cstrBin)$ has a solution. Let $I[\{N\} \cup X]$ be such a solution. 
By Lemmas~\ref{lem:Bmin} and~\ref{lem:Bmax}, the number of $C$-stretches $I[N]$ belongs to $[\underline{s}(X),\overline{s}(X)]$. 
($\Leftarrow$) Let $k\in[\underline{s}(X),\overline{s}(X)]\cap D(N)$ (not empty).
From Proposition~\ref{consistentK}, for any value $k$ in $[\underline{s}(X),\overline{s}(X)]$, $\exists v\in D(x_0)$ such that $k \in[\underline{s}(x_0,v),\overline{s}(x_0,v)]$.
Since ~\seqbin$(N,X,\cstrStretch,\cstrBin)$ is counting-continuous, there exists an instantiation of $X$ with $k$ $C$-stretches. 
By Definition~\ref{def:ctr} and since Lemmas~\ref{lem:Bmin} and~\ref{lem:Bmax} consider only  $\cstrBin$-coherent values, there is a solution of \seqbin$(N,X,\cstrStretch,\cstrBin)$ with $k$ $\cstrStretch$\nobreakdash-stretches. 
\end{proof}

\subsection{Necessary and Sufficient Filtering Condition} \label{cns}
Given \seqbin$(N,X,\cstrStretch,\cstrBin)$, Proposition~\ref{prop:consistency} can be used to filter the variable $N$ from variables in $X$. 
Propositions~\ref{prop:necessary} and~\ref{consistentK} ensure that every remaining value in $[\underline{s}(X),\overline{s}(X)]\cap D(N)$ is involved in at least one solution satisfying \seqbin.
We consider now the filtering of variables in $X$. 

\begin{proposition}\label{prop:consistencyprefix}
Given an instance of~\seqbin$(N,X,\cstrStretch,\cstrBin)$ which is counting-continuous and such that $\cstrBin$ is monotonic, let $v$ be a value in $D(x_i)$, $i\in[0,n-1]$. The two following propositions are equivalent: 
\begin{enumerate}
\item{$v$ is $\cstrBin$-coherent and $v$ is GAC with respect to~\seqbin}
\item{
$
\begin{bmatrix}
	\underline{p}(x_i,v)+\underline{s}(x_i,v)-1, \\
	\overline{p}(x_i,v)+\overline{s}(x_i,v)-1
\end{bmatrix}
\cap D(N) \neq \emptyset
$
}
\end{enumerate}
\end{proposition}
\begin{proof}
If $v$ is not $\cstrBin$-coherent then, by Definition~\ref{def:ctr}, $v$ is not GAC. 
Otherwise, $\underline{p}(x_i,v)$ (resp.  $\underline{s}(x_i,v)$) is the exact minimum number of \gstretches{} among $\cstrBin$\--coherent instantiations $I[x_0,x_1,\ldots,x_i]$ (resp. $I[x_i,x_{i+1},\ldots,x_{n-1}]$) such that $I[x_i]=v$. Thus, by Lemma~\ref{lem:Bmin}  (and its symmetrical for prefixes), the exact minimum number of $C$-stretches among  $\cstrBin$-coherent
instantiations $I[x_0,x_1,\ldots,x_{n-1}]$ such that $I[x_i]=v$ is $\underline{p}(x_i,v)+\underline{s}(x_i,v)-1$. Let ${\cal D}_{(i,v)} \subseteq {\cal D}$ such that all domains 
in ${\cal D}_{(i,v)}$ are equal to domains in  ${\cal D}$ except $D(x_i)$ which is reduced to $\{v\}$. 
We call $X_{(i,v)}$ the sequence of variables associated with domains in ${\cal D}_{(i,v)}$. 
By construction $\underline{p}(x_i,v)+\underline{s}(x_i,v)-1 = \underline{s}(X_{(i,v)})$. By a symmetrical reasoning,  $\overline{p}(x_i,v)+\overline{s}(x_i,v)-1  = \overline{s}(X_{(i,v)})$. By Proposition~\ref{prop:consistency}, the proposition holds. 
\end{proof}

The ``$-$ $1$" in expressions $\underline{p}(x_i,v)+\underline{s}(x_i,v)-1$ and $\overline{p}(x_i,v)+\overline{s}(x_i,v)-1$ prevents us from counting twice a $\cstrStretch$-stretch at an extremity $x_i$ of the two sequences $[x_0,x_1,\ldots,x_i]$ and  
$[x_i,x_{i+1},\ldots,x_{n-1}]$. 
\section{GAC Filtering Algorithm}\label{sec:algos}
Based on the necessary and sufficient filtering condition of Proposition~\ref{prop:consistencyprefix}, this section provides an implementation of the GAC filtering algorithm for a counting-continuous instance of~\seqbin$(N,X,C,B)$~with 
a monotonic constraint $B$. 

If $B \notin \{\leq,\geq,<,>,$\texttt{true}$\}$ then the total ordering $\prec$ entailing monotonicity of $B$ is not the natural order of integers. 
In this case, if $\prec$ is not  known, it is necessary to compute such an ordering with respect to all values in $\cup_{i \in [0,n-1]} (D(x_i))$, once before 
the first propagation of \seqbin. Consider that the two variables of $\cstrBin$ can take any value in $\cup_{i \in [0,n-1]} (D(x_i))$: 
Due to the inclusion of sets of supports of values (see Definition~\ref{def:subst}), the order remains the same when the domains of the variables constrained by $\cstrBin$ do not contain all values in $\cup_{i \in [0,n-1]} (D(x_i))$. 

To compute $\prec$, the following procedure can be used: Count the number of supports of each value, in $O(d^2)$ time (recall $d$ is the maximum domain size 
of a variable in $X$), and sort values according to the number of supports, in $O(|\cup_{i \in [0,n-1]} (D(x_i))| log(|\cup_{i \in [0,n-1]} (D(x_i))|))$ time. \\

Then, given the sequence of variables $X$, the algorithm is decomposed into four phases:
\begin{dingautolist}{192}
\item
Remove all non $\cstrBin$\--coherent values in the domains of $X$.
\item
For all values in the domains of $X$, compute the minimum and maximum number of \gstretches~of prefixes and suffixes.
\item
Adjust the minimum and maximum value of $N$ with respect to the minimum and maximum number of \gstretches{} of $X$.
\item
Using the result phase~\ding{193} and Proposition~\ref{prop:consistencyprefix}, prune the remaining $\cstrBin$\--coherent values.
\end{dingautolist}

With respect to phase~\ding{192}, recall that $\cstrBin$ is monotonic:  According to $\prec$, for any pair of 
variables $(x_i, x_{i+1})$, $\exists v_0$ in $D(x_i)$ such that $\forall v_j\in D(x_i)$, $v_j\neq v_0$, $v_j$ has 
a set of supports on $B(x_i,x_{i+1})$ included in the supports of $v_0$ on $B(x_i,x_{i+1})$. 
By removing from $D(x_{i+1})$ non supports of $v_0$ on $B(x_i,x_{i+1})$ in $O(|D(x_{i+1})|)$, all non $B$\--coherent values of $D(x_{i+1})$ with respect to $B(x_i, x_{i+1})$ are removed. 
By repeating such a process in the two directions (starting from the pair $(x_{n-2},x_{n-1})$ and from the pair $(x_0,x_1$)), all non $B$\--coherent values can be removed from domains in $O(\sumD)$ time complexity. 

To achieve phase~\ding{193} we use Lemmas~\ref{lem:Bmin} and ~\ref{lem:Bmax} and their symmetrical formulations for prefixes. 
Without loss of generality, we focus on the minimum number of $\cstrStretch$-stretches of a value $v_j$ in the domain of a variable $x_i$, $i<n-1$, thanks to Lemma~\ref{lem:Bmin}. Assume that for all $w \in D(x_{i+1})$, $\underline{s}(x_{i+1},w)$ has been computed. If there is no particular property on $\cstrStretch$, the supports $S_j \in D(x_{i+1})$ of $v_j$ on $\cstrStretch(x_i, x_{i+1}) \wedge \cstrBin(x_i, x_{i+1})$ and the subset $\neg S_j \in D(x_{i+1})$ of non-supports  of $v_j$ on $\cstrStretch(x_i, x_{i+1})$ which satisfy $\cstrBin$  have to be scanned, in order to determine for each set a value $w \in S_j$ minimizing $\underline{s}(x_{i+1},w)$ and a value $w' \in \neg S_j$ minimizing $\underline{s}(x_{i+1},w')+1$. This process takes $O(|D(x_{i+1})|)$ for each value, leading to $O(d^2)$ for the whole domain. Since all the variables need to  be scanned and for all the values in domains the quantities are stored, phase~\ding{193} takes $O(nd^2)$ in time, and $O(\sumD)$ in space. 

Phases~\ding{194} and~\ding{195} take $O(\sumD)$ time each since all the domains have to be scanned. By Proposition~\ref{prop:consistencyprefix}, all the non-GAC values have been removed after this last phase.  \\

If $B \in \{\leq,\geq,<,>,$\texttt{true}$\}$, $\prec$ is known. The worst-case time and space results come from Phase~\ding{193}.  The bottleneck stems from the fact that, when a domain $D(x_i)$ is scanned, the minimum and maximum 
number of $C$-stretches of each value are computed from scratch, while an incremental computation would avoid to scan $D(x_{i+1})$ for each value in $D(x_i)$. 
This observation leads to Property~\ref{prop:complexity}. 
Again, we focus on the minimum number of $\cstrStretch$-stretches on suffixes. Other cases are symmetrical. 

\begin{notation}\label{not:2}
Given \seqbin$(N,X,\cstrStretch,\cstrBin)$, $x_i \in X$, $0 \leq i < n$ and a value $v_j \in D(x_i)$,  
if $i<n-1$, let $V_j$ denote 
the set of integer values such that a value $s(v_j,w) \in V_j$ corresponds to each $w \in D(x_{i+1})$ and is equal to: 
\begin{itemize}
\item $\underline{s}(x_{i+1},w)$ if and only if $w \in  S_j$ 
\item $\underline{s}(x_{i+1},w)+1$ if and only if $w \in \neg S_j$
\end{itemize} 
\end{notation}
Within notation~\ref{not:2}, the set $V_j$ corresponds to the minimum number of stretches 
of values in $D(x_{i+1})$ increased by one if they are non supports of value $v_j$ with respect to $\cstrStretch$.

\begin{property}\label{prop:complexity}
Given a counting-continuous instance of~\seqbin$(N,X,\cstrStretch,\cstrBin)$ such that $B \in \{\leq,\geq,<,>,$\texttt{true}$\}$ and $x_i \in X$, $0 \leq i < n-1$,  
if the computation of $\min_{w \in D(x_{i+1})} (s(v_j,w))$ for all $v_j \in D(x_i)$  can be performed in $O(|D(x_{i+1})|)$ time 
then GAC can be achieved on \seqbin~in $O(\sumD)$ time and  space complexity. 
\end{property}
\begin{proof} Applying Lemma~\ref{lem:Bmin} to the whole domain $D(x_i)$ takes $O(|D(x_{i+1})|)$ time. Storing the minimum number 
of stretches for each value in $D(x_i)$ requires $O(|D(x_i)|)$ space.
Phase~\ding{193} takes $O(\sumD)$ space and $O(\sumD)$ time. 
\end{proof}

When they are represented by a counting-continuous instance of~\seqbin, the practical constraints mentioned in the introduction satisfy a condition  
that entails Property~\ref{prop:complexity}: Given $x_i$, it is possible to compute in $O(|D(x_{i+1})|)$ the quantity $\min_{w \in D(x_{i+1})}(s(v_0,w))$ for a 
 first value $v_0 \in D(x_i)$ and then,  following the natural order of integers, to derive with a constant or amortized time complexity the quantity for the next value $v_1$, and then the quantity for the next value $v_2$, and so on. 
Thus, to obtain GAC in $O(\sumD)$ for all these constraints, we specialize Phase~\ding{193} in order to exploit such a property. We now detail how to proceed. 

With respect to the constraints mentioned in the introduction corresponding to instances  of~\seqbin~which are not counting-continuous, the same time complexity can be reached but the algorithm does not enforces GAC. 

Thus, when \seqbin~represents \CH, \SM
~or \LNV, computing $\min_{w \in D(x_{i+1})}(s(v_0,w))$ for the minimum value $v_0 = \min(D(x_i))$ (respectively the maximum value) can be performed by scanning 
the minimum number of $\cstrStretch$-stretches of values in $D(x_{i+1})$. 

We now study for \CH, \SM
~and \LNV~how to efficiently compute the value $\min_{w \in D(x_{i+1})} (s(v_k,w))$ of $v_k \in D(x_i)$, either directly or from the previous value $\min_{w \in D(x_{i+1})} (s(v_{k-1},w))$,   in order to compute 
$\min_{w \in D(x_{i+1})} (s(v_j,w))$ for all $v_j \in D(x_i)$  in $O(|D(x_i)|)$ time and therefore 
achieve Phase~\ding{193} in $O(\sumD)$. 

\subsubsection{The \CH{} constraint}

Section~\ref{def:constraintClass} showed a reformulation of \CH$(N,X,\mathit{CTR})$ as \seqbin$(N',X,C,$\texttt{true}$) \wedge [N = N'-1]$, where $\cstrStretch$ is the opposite of $\mathit{CTR}$. 

%

$\--$ If $\cstrStretch$ is `$>$' (the principle is similar for `$\leq$','$\geq$' and '$<$'), the instance of \seqbin~is counting-continuous, because $B$ is $\texttt{true}$ and $C$ is monotonic. 
The monotonicity of $C$, with its corresponding total ordering $\prec$, guarantees that given three consecutive variables $x_{i-1}, x_i, x_{i+1}$ and $v_1 \in D(x_{i-1}), v_2 \in D(x_{i}), v_3 \in D(x_{i+1})$, if $(v_1,v_2)$ and $(v_2, v_3)$ both violate $C$, then we necessarily have $v_1 \succ v_2 \succ v_3$. 
Therefore, changing value $v_2$ by a new value $v'_2$ such that $v_1 \prec v'_2$ (to satisfy $C$) entails $v'_2 \succ v_2$, and thus still $v'_2 \succ v_3$ (which violates $C$). 
It is not possible to remove (or, symmetrically, to add) two violations of $C$ within an instantiation only by changing the value of one variable. 
The instance of \seqbin~is counting-continuous and thus the algorithm enforces GAC (by Proposition~\ref{prop:consistencyprefix}). 

To achieve step 3. in $O(D(x_i))$, we introduce two quantities $\mathit{lt}(v_j,x_{i+1})$ and $\mathit{geq}(v_j,x_{i+1})$  respectively equal to $\min_{w \in [\min(D(x_i)), v_j[}(\underline{s}(x_{i+1}, w))$ 
and $\min_{w \in  [v_j,\max(D(x_i))]}(\underline{s}(x_{i+1}, w))$. The computation is performed in three steps: 
\begin{enumerate}
\item Starting from $v_0 = \min(D(x_i))$, that is, the value having the smallest number of supports for $\cstrStretch$ on $x_{i+1}$, compute $\mathit{lt}(v_j,x_{i+1})$ in increasing order of $v_j$.  
Taking advantage that, given a value $v_{j-1} \in D(x_i)$ and the next value $v_{j} \in D(x_i)$, $[\min(D(x_i)), v_{j-1}[$ is included in $[\min(D(x_i)), v_{j}[$. 
Therefore, the computation of all $\min_{w \in [\min(D(x_i)), v_j[}(\underline{s}(x_{i+1}, w))$ can be amortized over $D(x_{i+1})$. 
The time complexity for computing $\mathit{lt}(v_j,x_{i+1})$ for all $v_j \in D(x_i)$ is in $O(|D(x_i)|+|D(x_{i+1})|)$. 
\item Similarly starting from $v_0 = \max(D(x_i))$,  compute incrementally $\mathit{geq}(v_j,x_{i+1})$ in decreasing order of $v_j$, in $O(|D(x_i)|+|D(x_{i+1})|)$. 
\item Finally, for each $v_j \in D(x_i)$,  $\min_{w \in D(x_{i+1})} (s(v_j,w))$ is equal to $\min(\mathit{lt}(v_j,x_{i+1}), \mathit{geq}(v_j,x_{i+1})+1)$. 
\end{enumerate}
Since step 3. takes $O(D(x_i))$, we get an overall time complexity for Phase~\ding{193} in $O(\sumD)$. 

$\--$ If $C$ is '$=$', '$\neq$', or $|x_i-x_{i+1}| \leq \mathit{cst}$ (the latter case corresponds to the \SM{} constraint), the filtering algorithm does not guarantees GAC because the corresponding instances of \seqbin~are not counting-continuous. Step 3. can also be performed in $O(D(x_i))$, leading to an overall time complexity for Phase~\ding{193} in $O(\sumD)$: 
\begin{itemize}
\item If $\cstrStretch$ is `$=$'  then 
for each $v_j \in D(x_i)$ there is a unique potential support for $\cstrStretch$ on $x_{i+1}$, the value $v_j$. 
Therefore, by memorizing once the value $\mathit{vmin}_1$ in $D(x_{i+1})$ which corresponds to the smallest minimum numbers of $\cstrStretch$-stretches on the 
suffix starting at $x_{i+1}$: $\forall v_j$,  $\min_{w \in D(x_{i+1})} (s(v_j,w))$ $=$ $\min(\underline{s}(x_{i+1}, v_j),\underline{s}(x_{i+1},\mathit{vmin}_1)+1)$, assuming $\underline{s}(x_{i+1}, v_j)=+\infty$ 
when $v_j \notin D(x_{i+1})$. 
\item If $\cstrStretch$ is `$\neq$' then for each $v_j \in D(x_i)$ there is a single non support. By memorizing the two values $\mathit{vmin}_1$ and $\mathit{vmin}_2$ 
which minimize the minimum numbers of $\cstrStretch$\nobreakdash-stretches on the 
suffix starting at $x_{i+1}$,  for any value $v_j$ $\min_{w \in D(x_{i+1})} (s(v_j,w))$ is equal to: $\min(\underline{s}(x_{i+1},\mathit{vmin}_1)+1, \underline{s}(x_{i+1},\mathit{vmin}_2))$ when $\mathit{vmin}_1 = v_j$, 
and $\underline{s}(x_{i+1},\mathit{vmin}_1)$ otherwise. 
\item \SM~is a variant of \CH$(N,X,\mathit{CTR})$, where $x_i$\,$\mathit{CTR}$\,$x_{i+1}$ is $|x_i-x_{i+1}| > \mathit{cst}$, $\mathit{cst} \in \mathbb{N}$, that can be reformulated as \seqbin$(N',X,C,$\texttt{true}$) \wedge [N = N'-1]$,
where $C$ is $|x_i-x_{i+1}| \leq \mathit{cst}$.
Assume $v_0 = \min(D(x_i))$ and we scan values in increasing order.
Supports of values in $D(x_i)$ for $|x_i-x_{i+1}|\leq\mathit{cst}$
define a set of sliding windows for which both the starts and the ends are increasing sequences (not necessarily strictly).
Thus, $\min_{w \in S_j} (s(v_j,w))$ can be computed for all $v_j$ in $D(x_i)$ in $O(|D(x_i)|)$ thanks to the
\emph{ascending minima algorithm}.\footnote{See http://home.tiac.net/{\small$\sim$}cri/2001/slidingmin.html} 
Given a value $v_j\in D(x_i)$ the set $\neg S_j$ of non supports of $v_j$ on $|x_i-x_{i+1}|\leq\mathit{cst}$
is partitioned in two sequences of values: a first sequence before the smallest support and a second sequence
after the largest support. While scanning values in $D(x_i)$ these two sequences correspond also to sliding windows
on which the ascending minima algorithm can also be used.
\end{itemize}


\subsubsection{The \LNV{} constraint}
It is represented by \seqbin$(N,X,=,\leq)$, which is counting-continuous (see~\cite{BeldiceanuHermenierLorcaPetit10} for more details). The algorithm enforces GAC.
Since $\cstrBin$ is not \texttt{true}, we have to take into account  $\cstrBin$ when evaluating $\min_{w \in D(x_{i+1})} (s(j,w))$ for each $v_j \in D(x_i)$. Fortunately, we can start from $v_0 = \max(D(x_i))$ and consider the decreasing order since $\cstrBin$ is `$\leq$'. In this case the set of supports on $B$ can only increase as
we scan $D(x_i)$. 
$\cstrStretch$ is `$=$', then for each $v_j \in D(x_i)$ there is a unique potential support for $\cstrStretch$ on $x_{i+1}$, the value $v_j$. 
We memorize once the value $\mathit{vmin}_1$ in $D(x_{i+1})$ which corresponds to the smallest minimum numbers of $\cstrStretch$-stretches on the 
suffix starting at $x_{i+1}$, {\em only on supports of the current value $v_j\in D(x_i)$ on $B$}. $\forall v_j$,  $\min_{w \in D(x_{i+1})} (s(v_j,w))$ $=$ $\min(\underline{s}(x_{i+1}, v_j),\underline{s}(x_{i+1},\mathit{vmin}_1)+1)$, assuming $\underline{s}(x_{i+1}, v_j)=+\infty$ 
when $v_j \notin D(x_{i+1})$. 
Since the set of supports on $B$ only increases,
 $\mathit{vmin}_1$ can be updated for each new value in $D(x_i)$ in $O(1)$.

\section{Related Work} \label{relatedWorks}
Using automata, \CH
~can be represented either by \RE~\cite{Pesant04} or by \CRE~\cite{DemasseyPR06}.
In the first case this leads to a GAC algorithm in $O(n^2d^2)$ time~\cite[pages 584--585, 1544--1545]{beldiceanu10}
(where $d$ denotes the maximum domain size).
In the second case the filtering algorithm of \CRE~does not achieve GAC.

Bessière \emph{et al.}~\cite{BessiereHebrardHnichKiziltanWalsh08} presented an encoding of the \CP~constraint with \SL$_2$. A similar reformulation can be used for encoding \seqbin$(N,X,\cstrStretch,\cstrBin)$. Recall that \SL$_j(C,[x_0, x_1,\ldots, x_{n-1}])$ holds if and only if $C(x_{ij}, \ldots, x_{ij+k-1})$ holds for $0\leq i \leq \frac{n-k}{j}$. Following a schema similar to the one proposed in Section 4 of Bessière \emph{et al.} paper, \seqbin$(N,X,\cstrStretch,\cstrBin)$ can be represented by adding a variable $N'$ and $n$ variables $[M_0, \ldots, M_{n-1}]$, with $M_0 = 0$ and $M_{n-1}=N'$. \seqbin$(N,X,\cstrStretch,\cstrBin)$ is then reformlated by
\SL$_2(C', [M_0,x_0, M_1, x_1, \ldots, M_{n-1},  x_{n-1}]) \wedge [N'=N-1] $,
where $C' = [\neg C(x_i, x_{i+1}) \wedge B(x_i, x_{i+1}) \wedge M_{i+1}=M_i+1] \vee [C(x_i, x_{i+1}) \wedge B(x_i, x_{i+1}) \wedge M_{i+1}=M_i]$. According to Section 6 of Bessière \emph{et al.} paper, GAC can be achieved thanks to a reformulation of \SL$_2$, provided a complete propagation is performed on $C'$, which is the case because $B(x_i, x_{i+1})$ and $C(x_i, x_{i+1})$ involve the same variables. The reformulation requires $n$ additional intersection variables (one by sub-sequence $[M_i, x_i]$), on which $O(n)$ compatibility constraints between pairs of intersection variables and $O(n)$ functional channelling constraints should hold. Arity of $C'$ is $k=4$ and $j=2$: the domain of an intersection variable contains $O(d^{k-j}) = O(d^2)$ values (corresponding to binary tuples), where $d$ is the maximum size of a domain. Enforcing GAC on a compatibility constraint takes $O(d^3)$ time, while functional channelling constraint take $O(d^2)$, leading to an overall time complexity $O(nd^3)$ for enforcing arc-consistancy on the reformulation, corresponding to GAC for \seqbin. To compare such a time complexity $O(nd^3)$ with our algorithm, note that $O(\sumD)$ is upper-bounded by $O(nd)$.

At last, some \emph{ad hoc} techniques can be compared to our generic GAC algorithm, \emph{e.g.}, a GAC algorithm
in $O(n^3 m)$ for \CH~\cite[page 57]{Hellsten04}, where $m$ is the total number of values in the domains of $X$.
Moreover, the GAC algorithm for \seqbin~generalizes to a class of counting constraints
the ad-hoc GAC algorithm for \LNV~\cite{BeldiceanuHermenierLorcaPetit10}
without degrading time and space complexity in the case where \seqbin~represents \LNV.

\section{Conclusion} \label{conclusion}
Our contribution is a structural characterization of a class of counting constraints 
for which we come up with a general polytime filtering algorithm achieving GAC under some conditions, 
and a characterization of the property which makes such an algorithm linear in the sum of domain sizes. 
A still open question is whether it would be possible or not to extend this class
(\emph{e.g.},~considering $n$\nobreakdash-ary constraints for \cstrBin~and \cstrStretch)
without degrading complexity or giving up on GAC,
in order to capture more constraints.

\bibliographystyle{named}
\bibliography{Inductive}
\end{document}